\documentclass[a4paper]{article}

\usepackage{graphicx}
\usepackage{subfigure} 
\usepackage{fancyhdr}
\usepackage{fancybox}
\usepackage{natbib}
\usepackage{hyperref}
\usepackage{xcolor}

\input{entete_emilie}

\begin{document}

\lhead{Emilie Morvant}
\rhead{Domain adaptation of weighted majority votes via PV-based self-labeling} 
\rfoot[Author's draft]{\thepage} 
\cfoot{} 
\lfoot[\thepage]{Author's draft}

\renewcommand{\headrulewidth}{0.4pt}  
\renewcommand{\footrulewidth}{0.4pt}

\title{Domain adaptation of weighted majority votes via perturbed variation-based self-labeling} 

 \author{
Emilie Morvant\footnote{Most of the work in this paper was carried out while Emilie Morvant was affiliated with Institute of Science and Technology (IST) Austria, 3400 Klosterneuburg.}
\and University Jean Monnet of Saint-Etienne, Laboratoire Hubert Curien, UMR CNRS 5516 
\and 18 rue du Professeur Beno\^it Lauras, 42000 Saint-Etienne, France 
\and \texttt{emilie.morvant@univ-st-etienne.fr}}

\maketitle

\begin{abstract}
In machine learning, the domain adaptation problem arrives when the test (target) and the train (source) data are generated from different distributions.
A key applied issue is thus the design of algorithms able to generalize on a new distribution, for which we have no label information.
We consider the specific PAC-Bayesian situation focused on learning classification models defined as a weighted majority vote over a set of real-valued functions. 
In this context, we present \mbox{PV-MinCq} a new framework that generalizes the non-adaptive algorithm MinCq.
\mbox{PV-MinCq} follows the next principle.
Justified by a theoretical bound on the target risk of the vote, we provide to MinCq a target sample labeled thanks to a perturbed variation-based self-labeling focused on the regions where the source and target marginals appear similar.
We also study the influence of our self-labeling, from which we deduce an original process for tuning the hyperparameters.
Finally, our experiments show very promising results. 
\end{abstract}

\centerline{\footnotesize {\bf Keywords: }  Machine learning; Classification; Domain adaptation; Majority vote; PAC-Bayes}

\section{Introduction}
Nowadays, due to the expansion of Internet a large amount of data is available. 
Then, many applications need to make use of supervised machine learning methods able to transfer knowledge from different information sources, which is known as transfer learning\footnote{See \citep{pan2010,Quionero-Candela2009} for surveys on transfer learning}.
In such a situation, we cannot follow the strong standard assumption in machine learning that supposes the learning and test data drawn from the same unknown distribution.
For instance, one of the tasks of the common  spam filtering problem consists in adapting a model from one user to a new one who receives significantly different emails.
This scenario, called domain adaptation (DA), arises when we aim at learning from a source distribution a well performing model on a different target distribution, for which one considers an unlabeled sample (or few labels)\footnote{The task with few target labels is sometimes referred to as semi-supervised DA, and the one without target label as unsupervised DA.}. In this paper we design a new DA framework when we have no target label. This latter situation is known to be  challenging \citep{BenDavid12}.

To address this kind of issues, several approaches exist in the literature\footnote{See \citep{margolis2011literature} for  a survey on DA.}. 
Among them, the instance weighting-based methods allow us to deal with the covariate-shift where the distributions differ only in their marginals ({\it e.g.} \cite{huang2007correcting}).
Another technique is to exploit self-labeling procedures.  However, it often relies on iterative and heavy self-labeling. For example, one of the reference methods is DASVM \citep{BruzzoneM10}. 
Concretely at each iteration, DASVM  learns a SVM classifier  from the labeled source examples, then some of them are replaced by target data auto-labeled with this SVM classifier\footnote{In DASVM, the self-labeled points correspond to those with the lowest confidence, and the deleted source points are those with the highest confidence.}. 
A third popular solution is to take advantage of a distance between distributions, with the intuition that we want to minimize this divergence while preserving good performance on the source data: If the distributions are close under this measure, then generalization ability may be ``easier'' to quantify.
The most popular divergences, such as the {\small $\Hcal\!\Delta\!\Hcal$}-divergence of \cite{BenDavid-NIPS07,BenDavid-MLJ2010} and the discrepancy of \cite{MansourMR08}, involve the disagreement between classifiers.
Although they lead to different analyses, they enhance to the same conclusion that is the disagreement between classifiers must be controlled while keeping a good source performance.
Obviously, other divergences for evaluating how much two distributions differ exist in the literature and could be investigated in a DA scenario.
For example, we can cite the perturbed variation (PV) \citep{PV} on which we will pay our attention for designing a non-iterative self-labeling process.  
This measure is based on the following principle: Two samples are similar if each instance of one sample is close to an instance of the other sample.

In this work, we investigate the special issue of PAC-Bayesian DA introduced by \cite{PBDA}, which focuses on learning  target weighted majority votes over a set of classifiers (or voters). Their analysis 
 stands in the class of approaches based on a divergence between distributions.
This latter, called the domain disagreement, has been justified by a tight bound over the risk of the majority vote---~the C-bound \citep{Lacasse07}~---and has the advantage to take into account the expectation of the disagreement between pairs of voters.
Although their theoretical analysis is elegant and well-founded, the algorithm derived is restricted to linear classifiers. 
We then intend to design a learning framework able to deal with weighted majority votes over real-valued voters in this PAC-Bayesian DA scenario. 
With this aim in mind and knowing the \mbox{C-bound} has lead to a simple and well performing algorithm for supervised classification, called MinCq \citep{MinCq}, we extend it to DA thanks to a non-iterative self-labeling.
Firstly, we propose a new formulation of the \mbox{C-bound} suitable for every self-labeling function (which associates a label to an example).
Then, we design such a function with the help of the empirical PV. Concretely, our PV-based self-labeling focuses on the regions where the source and target marginals are closer, then it labels the (unlabeled) target sample only in these regions (see Figure~\ref{fig:transfer}, in Section \ref{sec:transfer}). This self-labeled  sample is then provided to MinCq.
Afterwards, we highlight the influence of our self-labeling, and deduce an original validation procedure. 
Finally, our framework, named \mbox{PV-MinCq}, implies good and promising results, better than a nearest neighborhood-based self-labeling, and than other DA methods.

The rest of the paper is organized as follows.
Section~\ref{sec:background} recalls the PAC-Bayesian DA setting of \citep{PBDA}, and then MinCq and its theoretical basis 
in the supervised setting \citep{MinCq}. 
In Section~\ref{sec:contribution} we present \mbox{PV-MinCq}, our adaptive MinCq based on a PV-based self-labeling procedure. Before conlude, we experiment our framework on a synthetic problem in Section \ref{sec:experimentations}. 

\section{Notations and background}
\label{sec:background}
In this section, we first review the PAC-Bayesian setting in a non-adaptive setting, and then
 the results of \cite{PBDA} and \cite{MinCq}. 

\subsection{PAC-Bayesian setting in supervised learning}
We recall the usual setting of the PAC-Bayesian theory---introduced by \cite{McAllester99b}---which offers generalization bounds (and algorithms) for weighted majority votes over a set of real-valued functions, called voters.

Let $X \subseteq  \R^d$ be the input space of dimension $d$ and $Y = \{-1,+1\}$ be the output space, {\it i.e.} the set of possible labels.
$\PS$ is an unknown distribution over $X\times Y$, that we called a domain. $(\PS)^{\ms} = \bigotimes_{s=1}^{\ms} \PS$ stands for the distribution of a \mbox{$\ms$-sample}.  The marginal distribution of $\PS$  over $X$ is denoted by $\DS$. 
We consider $S = \{(\xbfs ,\ys)\}_{s=1}^{\ms}$ a \mbox{$\ms$-sample}  independent and identically distributed  ({\it i.i.d.}) according to $(\PS)^{\ms}$, commonly called the learning sample.
Let $\Hcal$ be a set of $n$ (bounded) real-valued voters such that: $\forall h  \in  \Hcal, \ h : X \to  \R$.
Given  $\Hcal$, the ingredients of the PAC-Bayesian approache are a prior distribution $\prior$ over $\Hcal$, a learning sample $S$  and a posterior distribution $\posterior$ over $\Hcal$.
Prior distribution $\prior$ models an {\it a priori} belief on what are the best voters from $\Hcal$, before observing the learning sample $S$.
Then, given the information provided by $S$, the learner aims at finding a posterior distribution $\posterior$ leading to a $\posterior$-weighted majority vote $\BQ$ over $\Hcal$ with nice generalization guarantees. $\BQ$ and its true and empirical risks are defined as follows.
\begin{definition}
Let $\Hcal$ be a set of real-valued voters. Let $\posterior$ be a distribution over $\Hcal$. The $\posterior$-weighted majority vote $\BQ$ (sometimes called the Bayes classifier) is:
\begin{align*}
\forall \xbf\in X,\ \BQ(\xbf) = \sign\left[\esp{h\sim \posterior} h(\xbf)\right].
\end{align*}
The true risk of $\BQ$ on a domain $\PS$ and its empirical risk\footnote{We express the risk with the linear loss since we deal with real-valued voters, but in the special case of $\BQ$ the linear loss is equivalent to the $0\!-\!1$-loss.} on a $\ms$-sample $S$ are respectively:
\begin{align*}
\RPS(\BQ) &%
= \frac{1}{2}\left(1-\esp{(\xbfs,\ys)\sim P} \ys\BQ(\xbfs)\right),\\
 \RS(\BQ) &
=  \frac{1}{2}\left(1-\frac{1}{\ms}\sum_{s=1}^{\ms}  \ys\BQ(\xbfs)\right).
\end{align*}
\end{definition}
Usual PAC-Bayesian analyses\footnote{Usual PAC-Bayesian analyses can be found in \citep{Mcallester03,Seeger02,Langford2005,catoni2007pac,germain2009pac}.} do not directly focus on the risk of $\BQ$, but bound the risk of the closely related stochastic Gibbs classifier $\GQ$. It predicts the label of an example $\xbf$ by first drawing a classifier $h$ from $\Hcal$ according to $\posterior$, and then it returns $h(\xbf)$. The risk of $\GQ$ corresponds thus to the expectation of the risks over $\Hcal$ according to $\posterior$:
\begin{align}
\RP(\GQ) &= \esp{h\sim \posterior} \RP(h) = \frac{1}{2}\left(1-\esp{(\xbfs,\ys)\sim P} \esp{h\sim \posterior} \ys h(\xbfs)\right). \label{eq:gibbs}
\end{align}
Note that it is well-known in the PAC-Bayesian literature that the deterministic $\BQ$ and the stochastic $\GQ$ are related by: 
\begin{align}
\label{eq:relation}
\RP(\BQ)\leq 2\RP(\GQ).
\end{align}

\subsection{PAC-Bayesian domain adaptation of the Gibbs classifier}
\label{pbda}
Throughout the rest of this paper, we consider the PAC-Bayesian DA setting introduced by \cite{PBDA}.
The main difference between supervised learning and DA is that we have two different domains over $X\times Y$: the source domain $\PS$ and the target domain $\PT$ ($\DS$ and $\DT$ are the respective marginals over $X$).
The aim is then to learn a good model on the target domain $\PT$ knowing that we only have label information from the source domain $\PS$. 
Concretely, in the setting described in \cite{PBDA}, we have a labeled source \mbox{$\ms$-sample} \mbox{$S=\{(\xbfs,\ys)\}_{t=1}^{\ms}$} {\it i.i.d.} from $(\PS)^{\ms}$ and a target unlabeled \mbox{$\mt$-sample} $T=\{\xbft\}_{t=1}^{\mt}$ {\it i.i.d.} from $(\DT)^{\mt}$.
One thus desires  to learn from $S$ and $T$ a  weighted majority vote with lowest possible expected risk on the target domain $\RPT(\BQ)$, {\it i.e.} with good generalization guarantees on $\PT$.
Recalling that usual PAC-Bayesian generalization bound study the risk of the Gibbs classifier, \cite{PBDA} have done an analysis of its target risk $\RPT(\GQ)$. 
Their main result is  the following theorem.
\begin{theorem}[{\small Theorem~ 4 of \citep{PBDA} applied to real-valued voters}]
\label{theo:pbda}
Let $\Hcal$ be a set of real-valued voters. For every distribution $\posterior$ over $\Hcal$, we have:
\begin{align}\label{eq:pbda}
\RPT(\GQ) \leq \RPS(\GQ) + \des(\DS,\DT) + \lambda_\posterior,
\end{align}
where $\des(\DS,\DT)$ is the domain disagreement between the marginals $\DS$ and $\DT$, and is defined by:
\begin{align}
\label{eq:dis}
\des(\DS,\DT)\! =\! \left|\esp{(h,h')\sim\posterior^2}\!  \!    \left(  \esp{\xbf_t\sim\DT}\!\! \!  h(\xbf_t) h'(\xbf_t) - \! \!  \!  \esp{\xbf_s\sim\DS} \!\!\!  h(\xbf_s)h'(\xbf_s)   \right)  \right|,
\end{align}
 and
$\lambda_\posterior$ is related\footnote{In practice, we cannot compute $\lambda_\posterior$, since it depends greatly on the unavailable target labels. We then suppose that it is negligible. 
Thus, we do not develop this point here, but more details can be found in \citep{PBDA}.} to the true labeling on $\PS$ and  $\PT$.
\end{theorem}
Note that this bound reflects the usual philosophy in DA: It is well known  that a favorable situation for DA arrives when the divergence between the domains is small while achieving good source performance \citep{BenDavid-NIPS07,BenDavid-MLJ2010,MansourMR08}. 
\cite{PBDA} have then derived a first promising algorithm called PBDA for minimizing this trade-off between source risk and domain disagreement.
Although PBDA has shown the usefulness of PAC-Bayes for tackling DA, it remains specific to linear classifiers, it does not directly focus on the majority vote, and does not provide the best empirical results regarding to state-of-the-art methods.

In this paper, our goal is to tackle this drawbacks to propose a novel algorithm for learning an adaptive weighted majority vote over a set of real-valued voters.
 To do so, the point which calls our attention here is the domain disagreement of Equation~\eqref{eq:dis}.
Indeed, it finds its root in the theoretical bound (the C-bound \citep{Lacasse07}) over the (source) risk of the majority vote,  from which \citep{MinCq} have derived  an elegant and performing non-adaptive algorithm for learning a weighted majority vote over a set of real-valued voters (MinCq). 
We recall now these non-DA results, then we extend them to DA in Section~\ref{sec:newcbound}.

\subsection{MinCq a supervised algorithm for learning majority votes}
The classical relation between the stochastic $\GQ$ and the majority  vote $\BQ$ (Equation~\eqref{eq:relation}) can be very loose. To tackle this drawback, \cite{Lacasse07} and \cite{MinCq} have proven a recent tighter relation stated in the following in Theorem~\ref{theo:C-bound} (the \mbox{C-bound}). This result is based on the notion of $\posterior$-margin defined as follows.
\begin{definition}[\cite{MinCq}]
The $\posterior$-margin of an example $(\xbf,y)\in X\times Y$ realized on the distribution $\posterior$ of support $\Hcal$ is given by:
$ \esp{h\sim \posterior} y h(\xbf).$
\end{definition}
By definition of $\BQ$, it is easy to see that $\BQ$ correctly classifies an example $\xbfs$ if the $\posterior$-margin is strictly positive. Thus, under the convention that if $\ys \espdevant{h\sim \posterior} h(\xbfs) = 0$, then $\BQ$ commits an error on $(\xbfs,\ys)$, for every domain $\PS$ on $X\times Y$, we have:
\begin{align*}
\RPS(\BQ)  & = \Prob{(\xbfs,\ys)\sim \PS} \left(\esp{h\sim \posterior}  \ys h(\xbfs) \leq 0 \right).
\end{align*}
Knowing this, \cite{Lacasse07} and \cite{MinCq} have  proven the following C-bound over $\RP(\BQ)$ by making use of the Cantelli-Chebitchev inequality.
\begin{theorem}[\small The C-bound as expressed in \cite{MinCq}]
\label{theo:C-bound}
For all distribution $\posterior$ over  $\Hcal$, for all domain $\PS$ over $X\times Y$ of marginal (over $X$) $\DS$, if  $\esp{h\sim \posterior}\esp{(\xbfs,\ys)\sim \PS} \ys h(\xbfs)> 0$, then:
\begin{align*}
 \RPS(\BQ) \leq 1 - \frac{\left(\esp{h\sim \posterior}\esp{(\xbfs,\ys)\sim \PS} \ys h(\xbfs)\right)^2 }{\esp{(h,h')\sim \posterior^2}\esp{\xbfs\sim \DS} h(\xbfs)h'(\xbfs)}.
\end{align*}
\end{theorem}
The numerator of this bound corresponds in fact to the first moment of the \mbox{$\posterior$-margin} of $\BQ$ realized on $\PS$, which is related to the risk of the Gibbs classifier (Equation \eqref{eq:gibbs}). 
The denominator is the second moment of this \mbox{$\posterior$-margin}, which can be seen as a measure of disagreement between the voters from $\Hcal$ (the lowest this value is, the more the voters disagree) and  can be related to the domain disagreement (Equation~\eqref{eq:dis}).

In the supervised setting, \cite{MinCq} have then proposed to minimize the empirical counterpart of the C-bound for learning a good majority vote over $\Hcal$, justified by an elegant PAC-Bayesian generalization bound. 
Following this principle the authors have derived a quadratic program  called MinCq and described in Algorithm~\ref{mincq}.
Concretely, MinCq  learns a weighted majority vote by optimizing the empirical C-bound measured on the learning sample $S$: It minimizes the denominator, {\it i.e.} the disagreement (Equation~\eqref{eq:objective}), given a fixed numerator {\it i.e.} a fixed risk of the Gibbs classifier (Equation~\eqref{eq:mincq_constraint1}), under a particular regularization (Equation~\eqref{eq:mincq_constraint2})\footnote{For more technical details on MinCq please see \citep{MinCq}.}. Note that, MinCq has showed good performances on supervised classification tasks.\\
 \begin{algorithm}[t]
\begin{algorithmic}
\small
\INPUT{A $\ms$-sample $S\!=\!\{(\xbfs,\ys)\}_{s=1}^{\ms}\!\sim\! (\PS)^{\ms}$,  a set of $n$ voters $\Hcal=\{h_1,\ldots,h_n\}$, a desired margin $\mu\! >\! 0$}
\OUTPUT{$\BQ(\cdot)=\sign\left[ \mbox{\scriptsize$\displaystyle \sum_{j=1}^{n}$} \left(2\posterior_j - \tfrac{1}{n}\right)h_j(\cdot)\right]$}
\begin{align}
\label{eq:objective}
    \textbf{Solve}\ \ &\argmin{\POSTERIOR}\ \mathbf{\POSTERIOR}^T{\bf M}{\bf \POSTERIOR-A}^T{\bf \POSTERIOR},\\[-3mm]
\label{eq:mincq_constraint1} \textbf{s.t.}\ \ & \displaystyle \mathbf{m}^T{\POSTERIOR}=\frac{\mu}{2}+\frac{1}{2n\ms }\sum_{j=1}^{n}\sum_{s=1}^{\ms} \ys h_j(\xbfs)
,\\[-1mm]
\label{eq:mincq_constraint2}&\displaystyle \forall j \in \{1,\ldots,n\},\quad 0 \leq \posterior_j \leq \tfrac{1}{n},
\end{align}
{\normalsize where  $\POSTERIOR  = (\posterior_1,\ldots,\posterior_{n})^{\top}$ is a weight vector,  
and $\Mbf$ is the $n \times  n$ matrix formed by {\tiny $\displaystyle \sum_{s=1}^{\ms}$}$ \frac{h_j(\xbfs )h_{j'}(\xbfs )}{\ms}$ 
for $(j,j')\! \in\! \{1,\ldots,n\}^2$, 
{\footnotesize \begin{align*} 
\mbox{and}\ \mbf&=\left( \tfrac{1}{\ms} \mbox{\tiny$\displaystyle\sum_{s=1}^{\ms}$}\ys h_1(\xbfs ),\ldots, \tfrac{1}{\ms} \mbox{\tiny$\displaystyle\sum_{s=1}^{\ms}$}\ys h_{n}(\xbfs ) \right)^{\top},\\
\mbox{and}\ \Abf &= \left(
\mbox{\tiny$\displaystyle\sum_{j=1}^{n} \sum_{s=1}^{\ms} $}\frac{h_1(\xbfs )h_{j}(\xbfs )}{n\ms} ,\ldots,   \mbox{\tiny$\displaystyle\sum_{j=1}^{n}  \sum_{s=1}^{\ms} $} \frac{h_{n}(\xbfs )h_{j}(\xbfs )}{n\ms}\right)^{\top}
\end{align*}}}
\end{algorithmic}
\caption{MinCq($S,\Hcal,\mu$) 
\label{mincq}}
\end{algorithm}

The key point here is that, through a DA point of view, the \mbox{C-bound}, and thus MinCq, focus on the trade-off suggested by Theorem~\ref{eq:pbda}. 
Indeed, the definition of the domain disagreement (Equation~\eqref{eq:dis}) is related to the C-bound 
according to the following statement: If source and target risks of the Gibbs classifier are similar, then the source and target risks of the majority vote are similar when the deviation between the source and target voters' disagreement tends to be low.

We thus now propose to make use of the C-bound and MinCq for designing an original and general framework for learning a majority vote over a set of real-valued voters in a DA scenario.





\section{An adaptive MinCq}
\label{sec:contribution}

In this section, we introduce our new DA framework for learning a weighted majority vote over a set of real-valued voters.
In order to take advantage of the algorithm MinCq, we first extend the C-bound to the DA setting.

\subsection{A C-bound suitable for DA with self-labeling}
\label{sec:newcbound}
Given a labeling function $l:X\to Y$, which associates a label $y\in Y$ to an unlabeled (target) example $\xbft\sim\DT$, we propose to rewrite the C-bound as follows.

\begin{corollary}
\label{cor:C-bound}
For all distribution $\posterior$ over $\Hcal$, for all domain $\PT$ over $X\times Y$ of marginal (over $X$) $\DT$,  for all labeling functions $l:X\to Y$ such that $\esp{h\sim \posterior}\esp{\xbft\sim \DT} l(\xbft) h(\xbft)> 0$, we have:
\begin{align*}
 \RPT(\BQ) \leq 1 - \frac{\left(\esp{h\sim \posterior}\esp{\xbf\sim \DT} l(\xbft) h(\xbft)\right)^2 }{\esp{(h,h')\sim \posterior^2}\esp{\xbft\sim \DT} h(\xbft)h'(\xbft)}
 + \frac{1}{2} \Big|\esp{(\xbft,\yt)\sim \PT}\!\big(\yt-l(\xbft)\big)\Big|
&.
\end{align*}
\end{corollary}
\begin{proof} The result comes directly from:\\
\centerline{$
\displaystyle \left|\RPT(\BQ) - \RPThat(\BQ)\right| =\tfrac{1}{2} \big|\esp{(\xbft,\yt)\sim \PT}(\yt-l(\xbft))\big|,
$}
where: $\RPThat(\BQ) =  \frac{1}{2}\Big(1-\esp{\xbft\sim\DT} l(\xbft)\BQ(\xbft)\Big).$
\end{proof}
We can recognize the C-bound of Theorem~\ref{theo:C-bound} where the true label $\yt$ of an example $\xbft$ is substituted by $l(\xbf)$.
The term 
$\tfrac{1}{2}\ \left|\espdevant{(\xbft,\yt)\sim \PT}\left(\yt-l(\xbft)\right)\right|$ 
can be seen as a divergence between the true labeling and the one provided by $l$, since it computes the gap between the labeling function and the true labeling one: The more similar $l$ and the true labeling functions, the tighter the bound is. 
Note that generalization bounds provided by \cite{MinCq} are still valid.

With a DA point of view, it is important to note that only one domain appears in this bound. If we suppose this domain is the target one, it is required to compute a relevant labeling function by making use of the information carried by the source labeled sample $S$.
To tackle the issue of defining this labeling function, that we called a self-labeling function, we follow the intuition that given a labeled source instance $(\xbfs,\ys)\in S$, we want to transfer its label $\ys$ to an unlabeled target point $\xbft$ close to $\xbfs$.
We thus propose to investigate the perturbed variation (PV) \citep{PV}, a recent measure of divergence between distributions based on this intuition.
This gives rise, in the following, to a PV-based self-labeling function, then to a self-labeled target sample on which we can apply MinCq (justified by Corollary~\ref{cor:C-bound}).

\subsection{Adaptive MinCq via PV-based self-labeling}
\label{sec:transfer}

\begin{algorithm}[t]
\begin{algorithmic}
\small
\INPUT{ $S\!=\!\{\xbfs\}_{s=1}^{\ms}$, $T\!=\!\{\xbft\}_{t=1}^{\mt}$ are unlabeled samples, a radius  $\epsilon\! >\! 0$, a distance measure $d\!:\!X\!\times\! X\! \to\! \R^+\!$}
\OUTPUT{$\widehat{PV}(S,T,\epsilon,d)$}
\STATE 1. $\!G\!\leftarrow\! \Big(V\!=\!(A,B),E\Big)$, $A\!=\!\{\xbfs\!\in\! S\}$, $B\!=\!\{\xbft\!\in\! T\}$,  $e_{st}\!\in\! E$ if $d(\xbfs,\xbft)\leq \epsilon$
\STATE 2. $\!M_{ST}\!\leftarrow$  Maximum matching on $G$
\STATE 3. $\!(S_u,T_u)\!\leftarrow$  number of unmatched vertices in $S$, resp. in $T$\\
\STATE 4. Return $\widehat{PV}(S,T,\epsilon,d)=\tfrac{1}{2}\Big(\frac{S_u}{\ms}+\frac{T_u}{\mt}\Big)$
\end{algorithmic}
\caption{\footnotesize $\widehat{PV}(S, T, \epsilon, d)$ \label{PV}}
\end{algorithm}

Before designing our self-labeling, 
we recall the definition of the PV  proposed by \cite{PV}.
\begin{definition}[\cite{PV}]
Let $\DS$ and $\DT$ be two marginal distributions over $X$ and $M(\DS,\DT)$ be the set of all joint distributions over $X\times X$ with marginals $\DS$ and $\DT$. The PV {\it w.r.t.} a distance $d:X\times X\to \R^+$ and $\epsilon > 0$ is:
\begin{align*}
PV(\DS,\DT, \epsilon, d) = \inf_{\nu\in M(\DS,\DT)} \Prob{\nu}\left[d({\cal X},{\cal X'})> \epsilon\right],
\end{align*}
over all pairs $(\DS,\DT)\sim \nu$, such that the marginal of $\cal X$ (resp. $\cal X'$) is $\DS$ (resp. $\DT$). 
\end{definition}
In other words, two samples are similar if every target instance is close to a source instance.
Note that this measure is consistent and  its empirical counterpart $\widehat{PV}(S,T,\epsilon,d)$ can be efficiently computed by a maximum graph matching procedure described in Algorithm~\ref{PV} \citep{PV}.

In our self-labeling goal, we make use of the maximum graph matching $M_{ST}$  computed at step $2$ of  Algorithm~\ref{PV}. Concretely, we label the unlabeled target examples from  $T$ thanks to $M_{ST}$, with the intuition that if $\xbft\in T$ belongs to a pair $(\xbft,\xbfs)\in M_{ST}$, then $\xbft$ is affected by the true label $\ys$ of $\xbfs$. Else, we remove $\xbft$ from $T$. The self-labeled sample $\widehat{T}$ constructed is:
$$\widehat{T} = \{(\xbft,\ys) :(\xbft,\xbfs)\! \in\! M_{ST}, \xbft\in T, (\xbfs,\ys)\in S\}.$$
Actually, we restrict the adaptation to region where the source and target marginals coincide under $d$.
Then we provide $\widehat{T}$ to MinCq. Our PV-based self-labeling is illustrated on Figure~\ref{fig:transfer}, our  framework, called \mbox{PV-MinCq}, is presented in Algorithm~\ref{PV-MinCq}.

\begin{figure}[t]
\centering
\includegraphics[width=0.7\columnwidth]{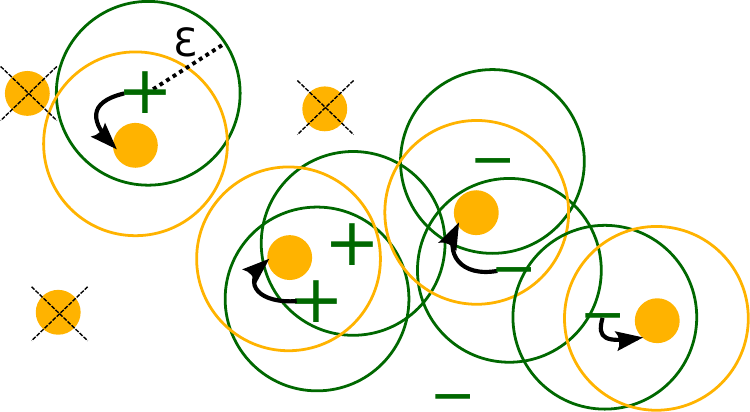}
\caption{Illustration of the PV-based self-labeling. The labeled source examples are in (dark) green, the unlabeled target examples are in (light) orange. The circles are the candidates for the matching. The arrows correspond to the matched points $M_{ST}$, and thus to the label transfer: The unmatched target examples are removed from the target sample. Note that the unmatched source and target samples indicate the PV.\label{fig:transfer}}
\end{figure}

\begin{algorithm}[t]
\begin{algorithmic}
\small
\INPUT{$S\!=\!\{(\xbfs,\ys)\}_{s=1}^{\ms}$ a source sample, $T\!=\!\{\xbft\}_{t=1}^{\mt}$  a target sample,  a set of voters $\Hcal$, a  margin $\mu\!>\!0$, a radius $\epsilon\!>\!0$,  a distance  $d\!:\!X\!\times\! X\! \to\! \R^+$}
\OUTPUT{$\BQ(\cdot)$}
\STATE 1. $M_{ST}$ $\leftarrow$ Step 1. and 2. $\widehat{PV}(S, T, \epsilon, d)$
\STATE 2. $\widehat{T} \leftarrow \{(\xbft,\ys)\! :\!(\xbft,\xbfs)\! \in\! M_{ST}, \xbt\!\in\! T, (\xbfs,\ys)\!\in\! S\}$
\STATE 3. return MinCq($\widehat{T},\Hcal,\mu$)
\end{algorithmic}
\caption{PV-MinCq($S,T,\Hcal,\mu,\epsilon,d$) \label{PV-MinCq}}
\end{algorithm}

\subsection{Analysis of the PV-based self-labeling}
\label{sec:analyse}
In this section, we discuss the impact of our PV-based self-labeling and the choice of the distance $d$.
Given a DA task, we first define the notion of a good distance.
\begin{definition}
\label{def:goodd}
Given a set of voters $\Hcal$ and $\epsilon>0$, a distance $d:X\times X\to \R^+$ is  $\epsH $-good for the DA task from $\PS$ to $\PT$, if there exists $\epsH  \geq 0$ such that:\\[1mm]
\centerline{
$\displaystyle \epsH  = \max_{h\in\Hcal,\ (\xbft,\xbfs)\sim \DS\times\DT,\ d(\xbft,\xbfs)\leq \epsilon} \left|
h(\xbfs)-h(\xbft) \right|
.$
}
\end{definition}
Put into words, we want the following natural property: If  $\xbft$ and $\xbfs$ are close under $d$, then for every voters in $\Hcal$ the deviation between the returned values $h(\xbfs)$ and $h(\xbft)$ is low.

Given a set of voters $\Hcal$, a fixed $\epsilon >0$ and a $\epsH $-good distance $d:X\times X\to \R^+$, we consider the matching $M_{ST}$ computed at step $2$ of Algorithm~\ref{PV}.
By definition, for every $(\xbft,\xbfs) \in M_{ST}$, $\xbft$ and $\xbfs$ share the same label $\ys$ and we have $d(\xbft,\xbfs) \leq \epsilon$.
We now  study  the influence of $d$ and $\epsH $ on the PAC-Bayesian DA bound of Theorem~\ref{theo:pbda}  restricted to  $M_{ST}$. We need of the following notations.
The source and the target subsamples associated to $M_{ST}$ are respectively:
\begin{align*}
\widehat{S} &= \{(\xbfs,\ys) :(\xbft,\xbfs)\! \in\! M_{ST}, \xbft\in T, (\xbfs,\ys)\in S\},\\
\widehat{T} &= \{(\xbft,\ys) :(\xbft,\xbfs)\! \in\! M_{ST}, \xbft\in T,  (\xbfs,\ys)\in S\}.
\end{align*}
Firstly, we bound the deviation between the risks of $\GQ$  on $\Shat$ and $\That$.
For all $\posterior$ on $\Hcal$, for every pair $(\xbft,\xbfs) \in M_{ST}$ we have:
\begin{align*}
 \left| \tfrac{1}{2}\left(1-  \ys \esp{h\sim\posterior} \! h(\xbft)\right) -  \tfrac{1}{2}\left(1-  \ys \esp{h\sim\posterior}\!   h(\xbfs)\right)\right|
=   \tfrac{1}{2}\, \left| \esp{h\sim\posterior}\!  (h(\xbft) -  h(\xbfs))\right|\\
\leq \tfrac{1}{2}\, \esp{h\sim\posterior}\! \left|  h(\xbft) -  h(\xbfs)\right| =   \tfrac{1}{2}\, \esp{h\sim\posterior}\!  \epsH   =  \tfrac{1}{2}\, \epsH.&
\end{align*}
Then, we have:
$\displaystyle \left| \RThat(\GQ) - \RShat(\GQ)\right| \leq    \tfrac{1}{2}\, \epsH .$\\[1mm]
Thus  the empirical risk of the Gibbs classifier on the source subsample $\Shat$ and the one on the self-labeled target sample $\That$ differ at most by $\tfrac{1}{2}\epsH $.
Hence the lower $\epsH $, the closer the risks are, and minimizing $\RThat(\GQ)$ is equivalent to minimize $\RShat(\GQ)$.

Secondly, similarly to the risks, we can bound the deviation between the voters' disagreement on $\Shat$ and $\That$.
For every $\posterior$ on $\Hcal$ and for every $(\xbft,\xbfs)\in M_{ST}$ , we have:
\begin{align*}
&\left|\esp{(h,h')\sim\posterior^2} \left[h(\xbfs)h'(\xbfs) - h(\xbft)h'(\xbft)\right]\right|\\
\leq\  &\bigg|\esp{(h,h')\sim\posterior^2}   \Big[(\epsH  +  h(\xbft))(\epsH   + h'(\xbft))
- h(\xbft)h'(\xbft)\Big]\bigg|\\
= \  & \bigg| \epsH ^2 +2\esp{h\sim\posterior}  \epsH h(\xbft) \bigg|
\end{align*}

Then, the empirical domain disagreement between $\Shat$ and $\That$ can be rewritten by:
\begin{align*}
\des(\Shat,\That)
 =\ & \left|\esp{(h,h')\sim\posterior^2} \esp{(\xbft,\xbfs)\in M_{ST}} 
\left[h(\xbfs)h'(\xbfs) - h(\xbft)h'(\xbft)\right]\right|\\
\leq\ &\esp{(\xbft,\xbfs)\in M_{ST}}\left|\esp{(h,h')\sim\posterior^2} 
\left[h(\xbfs)h'(\xbfs) - h(\xbft)h'(\xbft)\right]\right|\\
\leq\ &\esp{(\xbft)\in \That}\left|\epsH ^2 +2 \esp{h\sim\posterior}  \epsH h(\xbft) \right|\\
\leq\ & \epsH \left(1 + 2\esp{(\xbft)\in \That} \left| \esp{h\sim\posterior}  h(\xbft) \right|\right)
\end{align*}
In this situation, the divergence $\des(\Shat,\That)$ between the two samples can be bounded by a term depending on the confidence of the majority vote over $\That$ and on $\epsH $.\\
These results suggest that we have to minimize $\epsH $, while keeping good performances on $\That$. This confirms the legitimacy of our framework which
(i) transfers labels from the source sample to the target one in order to move closer the source and target risks of the Gibbs classifier, (ii) and then applies MinCq for optimizing the voters disagreement on the target sample (given a fixed Gibbs classifier risk on the self-labels).
However, although we can choose $\epsilon$ (and thus $\epsH $) as small as we desire, a low $\epsilon$ implies a smaller matching $M_{ST}$ and an higher empirical associated PV. In this case, the size of $\That$ tends to decrease, and then the guarantees of the Gibbs classifier decreases. 
In order to avoid this behavior, we exploit this property in the next section for designing a way to tune the hyperparameters.

\subsection{Validation of the hyperparameters}
\label{sec:validation}
A last question concerns the selection of the hyperparameters $\mu$ and $\epsilon$.
Usually in DA, one can make use of a reverse/circular validation \citep{BruzzoneM10,Zhong-ECML10}, with the idea that if the domains are close/related then a reverse classifier, learned from the target data labeled with the current classifier, has to perform well on the source data (see Figure \ref{fig:reverse} for the intuition). 
However, for \mbox{PV-MinCq} our first step is to transfer the source labels. As we have seen previously,  our main goal is then to validate this transfer, and 
 the reverse validation appears less relevant.
\begin{figure}[t]
\centering
\includegraphics[width=0.77\columnwidth]{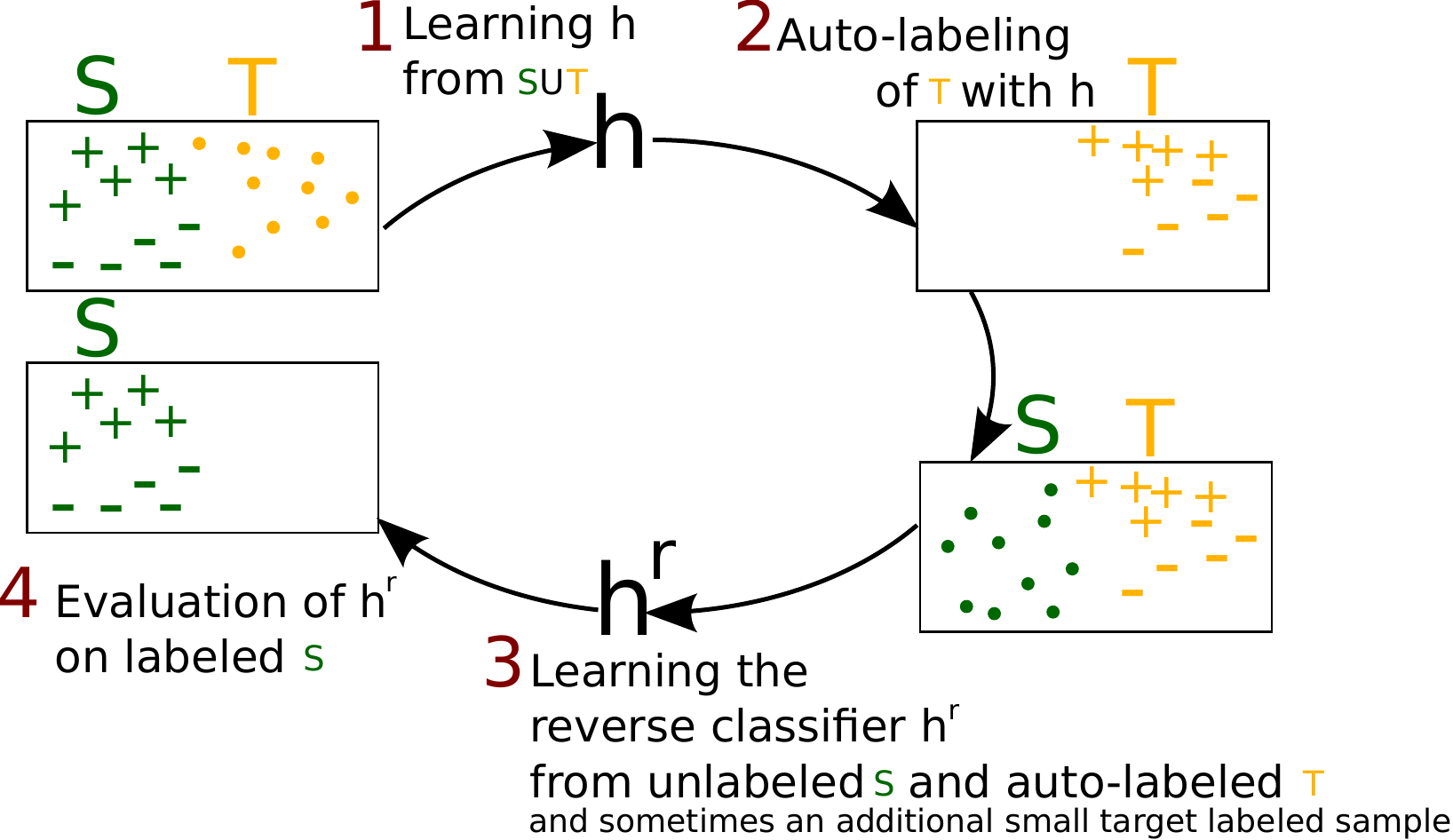}
\caption{The principle of the reverse validation.\label{fig:reverse}}
\end{figure}
We thus propose to deal with our analysis by making use of all the available information,{\it i.e} the original samples $S$ and $T$.
In this context,  on the one hand, since we have shown that the domain disagreement can be upper-bounded by a term depending on the PV-based self-labeling (and the confidence of the majority vote on the target self-labels), the PV between  $\DS$ and $\DT$ has to be controlled: The lower the PV, the more similar the samples are.
However, minimizing the PV regarding to $\epsilon$ can be easy: It is possible to find a high\footnote{{\it e.g.}, if $\epsilon$ equals to the highest distance between source and target example.} value for $\epsilon$, leading to a small PV.
On the other hand, to compensate this behavior, we thus have to control the performance. Indeed, the higher $\epsilon$ is, the higher the distance between source and target examples of the same pair (from $M_{ST}$) is.
This implies that these points are less similar, which tends to increase the deviation between the source and target risks, and could imply a loss of performances on the original source sample. 
Therefore, a relevant PV-based self-labeling corresponds to the one enable to optimize the following trade-off:\\[1.5mm]
\centerline{$\RS(\BQ) + \widehat{PV}(S,T,\epsilon,d),$}\\[1.5mm]
where $\RS(\BQ)$ is the empirical risk on the source sample,
and $\widehat{PV}(S,T,\epsilon,d)$ is the empirical PV between $S$ and $T$.
It is worth noting that this process can also be seen in connection with the philosophy of DA: We want to minimize the divergence between the  domains while keeping good source performance. 

Concretely, for every set of possible parameters $(\mu, \epsilon)$ and given $k$-folds on the source sample ($S\!=\!\cup_{i=1}^kS_i$), \mbox{PV-MinCq} learns a majority vote  $\BQ$ from  the $k\!-\!1$ labeled folds of $S$ (and $T$). Then, we evaluate $\BQ$ on the last $k^{th}$ fold. Its empirical risk corresponds then to the mean of the error over the $k$-folds:\\
\centerline{
$\displaystyle \RS(\BQ)=\frac{1}{k}\sum_{i=1}^k \Risk_{S_i}(\BQ),
$}\\
and $\widehat{PV}(S,T,\epsilon,d)$ is computed by Algorithm~\ref{PV}.

\section{Experimental results}
\label{sec:experimentations}

In this section, we evaluate  our framework \mbox{PV-MinCq} for learning a  vote over a set of Gaussian kernels defined from the learning sample.
We compare it to the following methods: 
\\ \textbullet$\ $  SVM only from the source sample, {\it i.e.} without adaptation;
\\ \textbullet$\ $  MinCq \citep{MinCq} only from the source sample;
\\ \textbullet$\ $  TSVM, the semi-supervised transductive-SVM\footnote{TSVM has not been proposed for DA scenario, but provides in general very interesting results in DA.}, \citep{Joachims99} learns from the two domains; 
\\ \textbullet$\ $  DASVM \citep{BruzzoneM10}, an iterative self-labeling DA algorithm;
\\ \textbullet$\ $  DASF \citep{DASF12}, a DA algorithm minimizing a trade-off between a  divergence and a source risk based on the analysis of \cite{BenDavid-NIPS07};
\\ \textbullet$\ $  PBDA \citep{PBDA}, the PAC-Bayesian DA algorithm for minimizing the bound of Theorem~\ref{theo:pbda};
\\ \textbullet$\ $  \mbox{PV-SVM}, for which we compute the self-labeling of the target data as for \mbox{PV-MinCq}, and then we apply a classical SVM on these self-labeled data;
\\ \textbullet$\ $  \mbox{NN-MinCq} that uses a $k$-NN based self-labeling: We label a target point with a $k$-NN classifier of which the prototypes comes from the source sample ($k$ is tuned).\\
To compute the \mbox{PV-based}  self-labeling, we make use of the euclidean distance.
Each parameter is selected with a grid search via a classical $5$-folds cross-validation for SVM, MinCq and TSVM, a reverse $5$-folds cross-validation 
for DASVM, DASF, PBDA and \mbox{NN-MinCq}, and the \mbox{PV-based} validation procedure described in Section \ref{sec:validation} for \mbox{PV-SVM} and \mbox{PV-MinCq}.

\begin{figure}[t]
\centering
\includegraphics[width=0.23\columnwidth]{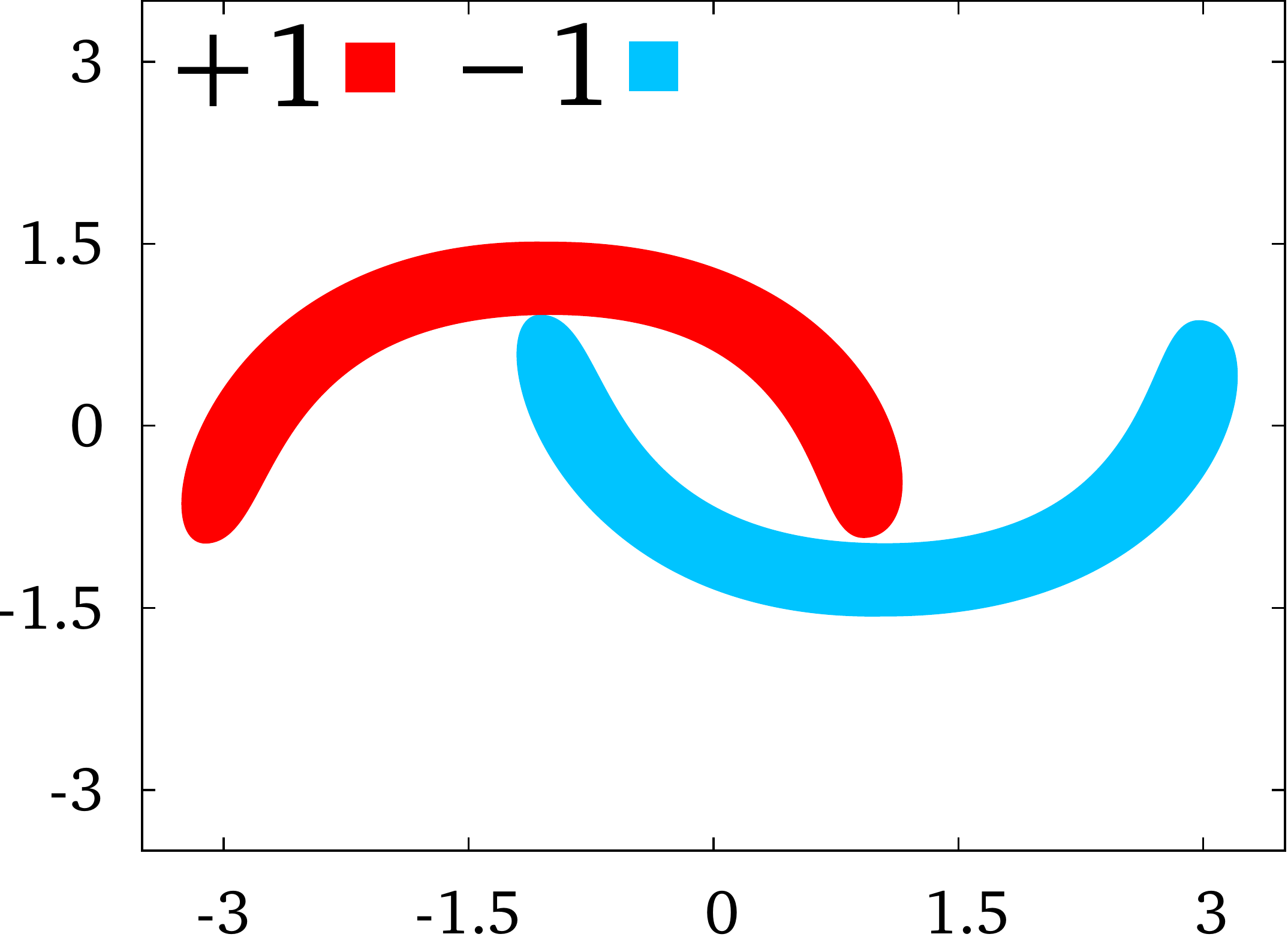}\hfil\includegraphics[width=0.23\columnwidth]{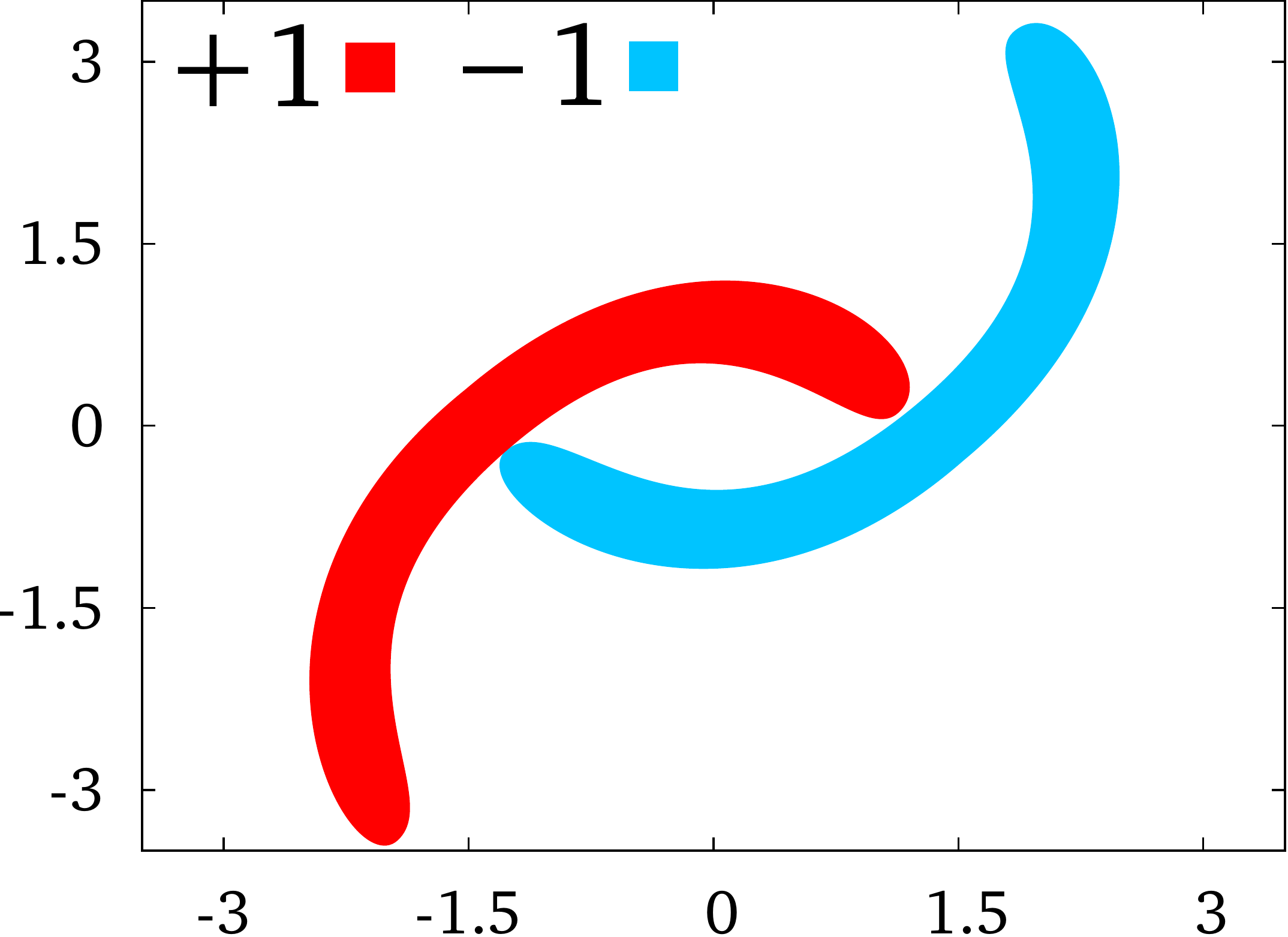}\includegraphics[width=0.23\columnwidth]{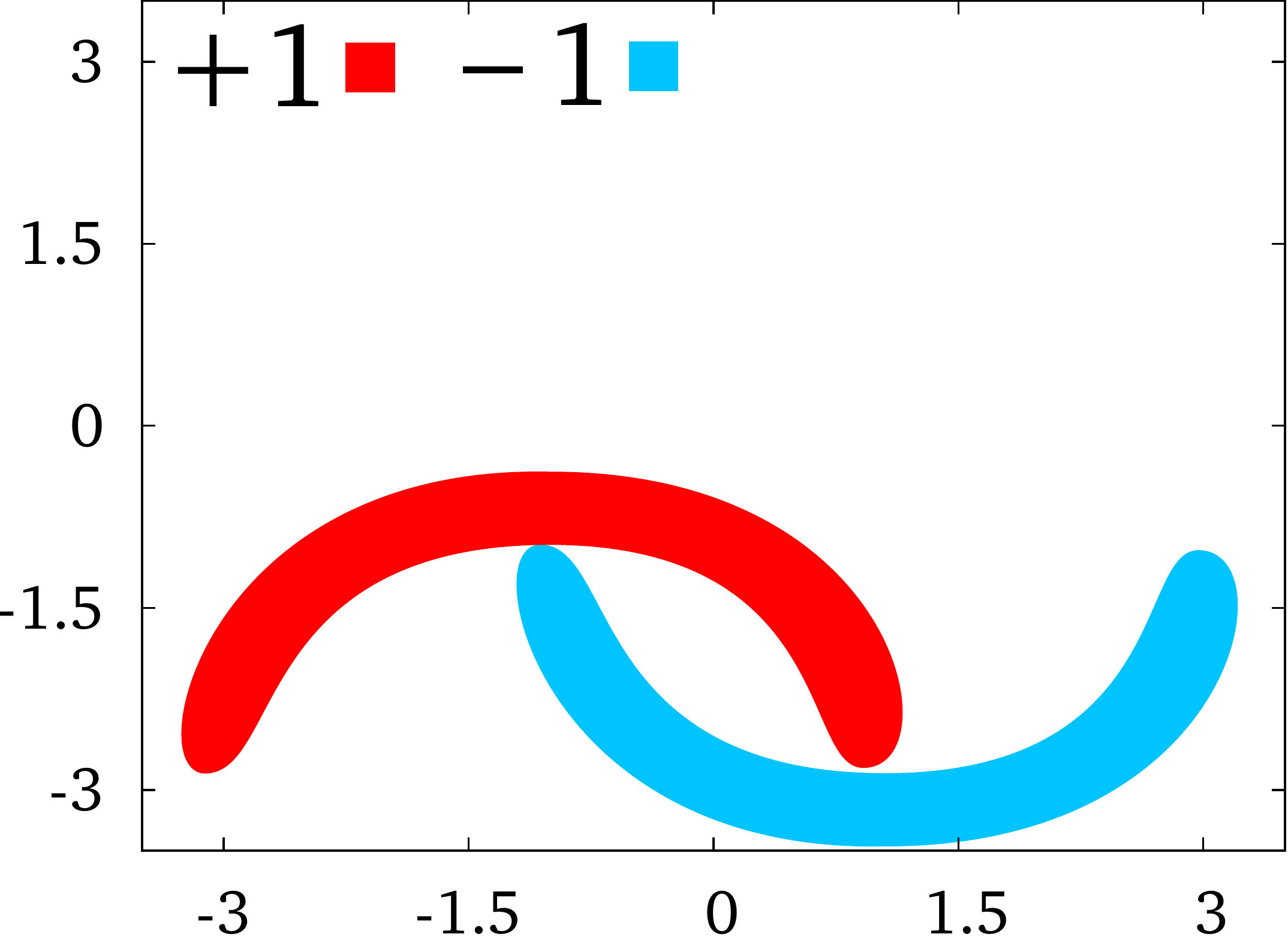}
\caption{On the left: the source domain. On the right: a target domain with a $40\degree$ rotation angle, and the translated target domain.\label{fig:moon}}
\end{figure}
\begin{table}[t]
\scriptsize \centering
\caption{\label{tab:resultats}Average accuracy results on ten runs for the seven rotations, and for the translation ({\it trans.}). NN-MinCq implies no result for the latter since, in this case, the self-label is the same for all target examples.}
 \begin{tabular}{@{}c@{}||c|c|c|c|c|c|c||@{}c@{}}
          \toprule
          Rot. angle&\hspace*{-0.2cm}$20\degree$\hspace*{-0.2cm}&\hspace*{-0.2cm}$30\degree$\hspace*{-0.2cm}&\hspace*{-0.2cm}$40\degree$\hspace*{-0.2cm}&\hspace*{-0.2cm}$50\degree$\hspace*{-0.2cm}&\hspace*{-0.2cm}$60\degree$\hspace*{-0.2cm}&\hspace*{-0.2cm}$70\degree$\hspace*{-0.2cm}&\hspace*{-0.2cm}$80\degree$\hspace*{-0.2cm}&$\,${\it trans.}\\
          \bottomrule
          \toprule
 SVM            &  $89.6$     & $76$     &  $68.8$    &  $60$   & $47.2$  &  $26.1$    & $19.2$  & $50.6$\\
          \midrule
          MinCq & $92.1$ & $78.2$ & 	$69.8$ &	$61$ &$ 50.1$ &$40.7$ &	$32.7$ & $50.7$ \\
\bottomrule
          \toprule
          TSVM          &   $\mathbf{100}$     &  $78.9$    &  $74.6$    &  $70.9$   & $64.7$ & $21.3$    &  $18.9$   & $94.9$  \\ 
\midrule
DASVM  & $\mathbf{100}$ & $78.4$& $71.6$&$66.6$ &$61.6$ & $25.3$ & $21.1$ & $50.1$ \\ 
\midrule
PBDA & $90.6$  & $89.7$ & $77.5$ & $58.8$ & $42.4$ & $37.4$ & $39.6$ & $85.9$ \\ 
\midrule
         DASF          &  $98.3$ & $92.1$ & $83.9$ & $70.2$ & $54.7$ &$ 43$ &$38.9$ & $82.8$\\ 

\midrule
\mbox{PV-SVM}  &     $94.2$     & $82.5$     &  $75.1$    &  $67.7$   & $55.2$  &  $43.6$    & $30.3$ & $97.1$
\\
\midrule
\mbox{NN-MinCq}$\,$& $97.7$ & $83.7$ & $77.7$ & $69.2$ & $58.1$ & $47.9$ & $42.1$ & $\varnothing
$\\
\bottomrule
          \toprule
\mbox{PV-MinCq} & $99.9$ & $\mathbf{99.7}$ & $\mathbf{99}$ & $\mathbf{91.6}$ & $\mathbf{75.3}$ & $\mathbf{66.2}$ &  $\mathbf{58.9}$ & $\mathbf{97.4}$ \\ 
\bottomrule
\end{tabular}
\end{table}

We tackle the  binary classification task called ``inter-twinning moon''.
The source domain is problem where each moon corresponds to one label (see Figure~\ref{fig:moon}).
We consider seven different target domains by rotating anticlockwise the source one according to seven rotations angles from $20\degree$ to $80\degree$. The higher the angle, the more difficult the adaptation is.
We also consider one target domain as a translation of the source one.
We randomly generate $150$ positives examples and $150$ negatives examples for each domain. 
To estimate the generalization error of our approach, each algorithm is evaluated on an independent test set of $1,500$ target instances.
Each DA task is repeated ten times. We report the average correct classification percentage on Table~\ref{tab:resultats}. We make the following remarks.

First, \mbox{PV-MinCq} outperforms on average the others, and appears more robust to change of density (\mbox{NN-MinCq} and MinCq appears also more robust).
We also observe that SVM, respectively \mbox{PV-SVM}, provides lower performance than MinCq, respectively \mbox{PV-MinCq}. 
These observations confirm the necessity of taking into account the voters' disagreement. 
Second, the \mbox{PV-based} labeling implies better results than the NN one. For the translation task the labels affected by the NN-based self-labeling are the same for every target example.
Unlike a \mbox{NN-based} labeling, using the matching implied by the PV appears to be a colloquial way to control the divergence between domains since it clearly focuses on the highest density region by removing the target points without matched source point, in other words on regions where the domains are close.
These results confirm that the PV coupled with MinCq provides a nice solution to tackle DA for learning a target majority vote.


\section{Discussion and future work}
\label{sec:conclusion}
We design a general PAC-Bayesian domain adaptation (DA) framework---\mbox{PV-MinCq}---for learning a target weighted majority vote over a set of real-valued functions.
To do so, \mbox{PV-MinCq} is based on MinCq, a quadratic program for minimizing the \mbox{C-bound} over the majority vote's risk by controlling the disagreement between voters known to be crucial in DA.
The idea is to focus on regions where the marginals are closer in order to transfer the source labels to the unlabeled target examples (only in these regions). Then we apply MinCq on these self-labeled points, justified by a new version of the \mbox{C-bound} formulated to deal with self-labeling functions.
We propose a the self-labeling process which has the originality to be defined thanks to the perturbed variation (PV)  between the source and target marginals.
Moreover, it has the clear advantage to be \mbox{non-iterative}, unlike usual self-labeling DA algorithms that are generally based on iterative procedures.
As a consequence, \mbox{PV-MinCq} is easier to apply.
Subsequently, we highlight the necessity of controlling the trade-off between low empirical PV and low source risk, that leads to an original hyperparameters selection.
Finally, the empirical results are promising, and raise to exciting directions.

 For instance, PV-MinCq could be useful for efficiently combining several  data descriptions such as in multiview or multimodality learning\footnote{{\it e.g.} a document can be represented by different descriptors.}.
Indeed, in such a situation one natural solution consists in (i) learning a classifier from each description, and (ii) learning\footnote{Sometimes referred as stacking or classifier fusion.} a majority vote over the learned classifiers. Thus, for adapting a majority vote from a source corpus to a target one, (ii) can be performed by \mbox{PV-MinCq}.

Given a DA task, another interesting direction is the design (or the learning) of a \mbox{$\epsH$-good} distance (or metric) $d$ to provide a specific self-labeling, allowing  a more accurate  computation of the PV. Indeed, our analysis  of the self-labeling suggests the requirement of a distance $d$ implying a pertinent measure of closeness in the semantic space involved by the voters.

Lastly, our results raise the question of the usefulness of the PV to learn shared features or points across domains ({\it e.g.} as done in \citep{GongGS13,LinAZ13}) by identifying which source samples are relevant for the target task. 


\bibliographystyle{model2-names}
\bibliography{main_pvmincq}

\end{document}